\documentclass{article} %
\usepackage{iclr2024_conference,times}
\usepackage{fullpage}
\usepackage{algorithm}
\usepackage{algorithmic}
\usepackage{amsthm}
\newtheorem{claim}{Claim}

\usepackage{amsmath,amsfonts,bm}

\def\eqref#1{equation~\ref{#1}}

\def\1{\bm{1}}

\def\eps{{\epsilon}}

\DeclareMathAlphabet{\mathsfit}{\encodingdefault}{\sfdefault}{m}{sl}
\SetMathAlphabet{\mathsfit}{bold}{\encodingdefault}{\sfdefault}{bx}{n}

\newcommand{\E}{\mathbb{E}}

\usepackage{hyperref}
\usepackage{url}
\usepackage{graphicx}
\usepackage{caption}
\usepackage{wrapfig}
\usepackage{cleveref}
\allowdisplaybreaks

\usepackage{siunitx}
\newcommand{\red}[1]{\textcolor{red}{#1}}

\title{SOAP: Improving and Stabilizing Shampoo using Adam}

\author{
Nikhil Vyas\thanks{Equal contribution. Correspondence to \texttt{vyasnikhil96@gmail.com}.}\\
Harvard University
\And
Depen Morwani$^*$\\
Harvard University
\And
Rosie Zhao\\
Harvard University
\And
Mujin Kwun\\
Harvard University
\And
Itai Shapira\\
Harvard University
\And
David Brandfonbrener\\
Kempner Institute at Harvard University
\And
Lucas Janson\\
Harvard University
\And 
Sham Kakade\\
Kempner Institute at Harvard University
}

\iclrfinalcopy %
\begin{document}

\maketitle

\begin{abstract}
There is growing evidence of the effectiveness of Shampoo, a higher-order preconditioning method, over Adam in deep learning optimization tasks. However, Shampoo's drawbacks include additional hyperparameters and computational overhead when compared to Adam, which only updates running averages of first- and second-moment quantities. This work establishes a formal connection between Shampoo (implemented with the 1/2 power) and Adafactor --- a memory-efficient approximation of Adam --- showing that Shampoo is equivalent to running Adafactor in the eigenbasis of Shampoo's preconditioner. This insight leads to the design of a simpler and computationally efficient algorithm: \textbf{S}hampo\textbf{O} with \textbf{A}dam in the \textbf{P}reconditioner's eigenbasis (SOAP).
With regards to improving Shampoo's computational efficiency, the most straightforward approach would be to simply compute Shampoo's eigendecomposition less frequently. 
Unfortunately, as our empirical results show, this leads to performance degradation that worsens with this frequency.
SOAP mitigates this degradation by continually updating the running average of the second moment, just as Adam does, but in the current (slowly changing) coordinate basis. Furthermore, since SOAP is equivalent to running Adam in a rotated space, it introduces only one additional hyperparameter (the preconditioning frequency) compared to Adam. We evaluate SOAP on language model pre-training, with experiments on 360m and 660m sized models. In the large batch regime, SOAP reduces the number of iterations by over 40\% and wall clock time by over 35\% compared to AdamW, with approximately 20\% improvements in both metrics compared to Shampoo. An implementation of SOAP is available at \url{https://github.com/nikhilvyas/SOAP/tree/main}.%
\end{abstract}

\section{Introduction}

With ever-increasing costs of LLM training, optimization efficiency has become a central question in the field of deep learning. Several recent works have tackled this challenge by addressing both the memory \citep{galore, 4bitshampoo} and compute \citep{anil2020scalable} footprint of optimizers. In Algoperf~\citep{dahl2023benchmarking}, a recent optimization efficiency benchmark, Shampoo \citep{gupta2018shampoo}, a second-order algorithm, outperformed all other submissions, including Adam \citep{adam}, reducing wall-clock time by 28\% \citep{algoperfresults}. Higher-order preconditioning has also been applied in large-scale training runs, such as Gemini-1.5 Flash \citep{gemini15}.

The success of Shampoo has drawn increasing attention from the deep learning community. Several works have explored ways to scale Shampoo by improving its memory and compute efficiency \citep{4bitshampoo, anil2020scalable, distributedshampoo}. Other research~\citep{whyshampoo} has examined the theoretical foundations of Shampoo and proposed minor adjustments (such as using power $1/2$ rather than $1/4$) that align with prior empirical findings~\citep{anil2020scalable}. Moreover, \citet{whyshampoo} also showed that Shampoo with the aforementioned modifications is close to the optimal Kronecker approximation of the Adagrad \citep{adagrad} optimizer.

Our first contribution is demonstrating that the variant of Shampoo proposed by \citet{whyshampoo} is equivalent\footnote{Given this connection, the results of~\citet{whyshampoo} can be interpreted as showing that the eigenbasis provided by Shampoo's preconditioner is close to the ``optimal'' basis for running Adafactor.} to running Adafactor \citep{adafactor, zhai} in the eigenbasis provided by Shampoo's preconditioner. This interpretation of Shampoo connects it to a broader family of methods (e.g.~\citep{ekfac}) that design second-order algorithms by running a first-order method in the eigenbasis provided by a second-order method. Building on this insight, we can explore a broader design space for combining first and second order methods. Many of our design choices are a synthesis of conceptual ideas from prior works of~\citet{ekfac,anil2020scalable,whyshampoo} as well as implementation ideas from works of~\citet{4bitshampoo,galore}.

Explicitly, we study SOAP (\textbf{S}hampo\textbf{O} with \textbf{A}dam in the \textbf{P}reconditioner's eigenbasis) an algorithm that runs AdamW in the eigenbasis provided by Shampoo.  Our main contributions are as follows:

\begin{itemize} 
\item We make a formal connection between the Shampoo and the Adafactor algorithm. This insight leads us to consider the SOAP algorithm, which runs AdamW in the preconditioned space provided by Shampoo. 
\item SOAP outperforms both Shampoo and Adam in language model pre-training tasks with model sizes 360m and 660m, even after extensive hyperparameter tuning of Shampoo. 
\item SOAP reduces the number of hyperparameters compared to Shampoo, resulting in only one additional hyperparameter compared to AdamW: preconditioning frequency. 
\item SOAP demonstrates greater robustness to large preconditioning frequency compared to Shampoo on language model pre-training tasks.
\end{itemize}

We should also note that while similar algorithmic variants have been discussed in the literature (e.g. see the appendix of~\citet{anil2020scalable}), we are the first to systematically evaluate it.

\textbf{Organization:} In~\Cref{sec:related}, we discuss related works. In~\Cref{sec:alg}, we start by showing an equivalence between Shampoo (with exponent 1/2) and running Adafactor in the eigenspace given by Shampoo, then with this equivalence as the starting point we describe SOAP. In~\Cref{sec:methods}, we provide our experimental methodology and in~\Cref{sec:language}, we compare the performance of AdamW, Shampoo and SOAP on language modeling tasks. In~\Cref{app:adafactor,app:theory-overhead} we discuss the the space and time complexity of SOAP and how it can be improved. In~\Cref{sec:long} we show that efficiency benefits of SOAP over AdamW are maintained for longer duration runs where \#tokens = 100 $\times$ model size.

\begin{figure}
    \centering
    \includegraphics[width=1\linewidth]{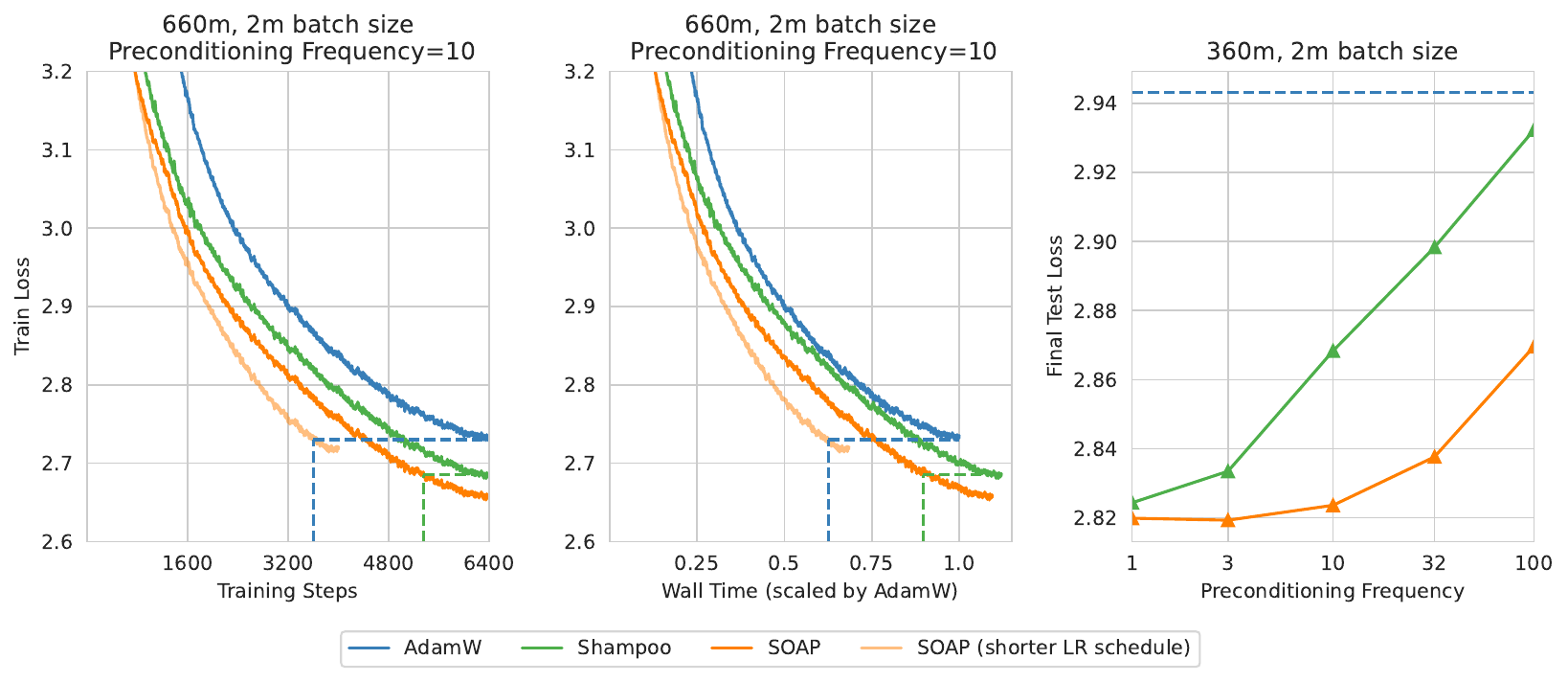}
    \caption{Comparing performance of tuned runs for AdamW, Shampoo (using DistributedShampoo~\citep{distributedshampoo} implementation) and SOAP. In left and middle figures, Shampoo and SOAP use a preconditioning frequency of 10. The "shorter LR schedule" plot is where we tuned the cosine decay so as to achieve the same terminal performance as AdamW. There we observe a $\geq 40\%$ reduction in the number of iterations and a $\geq 35\%$ reduction in wall clock time compared to AdamW, and approximately a 20\% reduction in both metrics compared to Shampoo. In the right figure we ablate preconditioning frequency and observe a slower degradation of performance of SOAP as compared to Shampoo. See~\Cref{sec:language} for a discussion of experimental results and ablation of batch size and~\Cref{sec:methods} for experimental methodology.}
    \label{fig:main}
\end{figure}

\section{Notation and Background}
We denote the weight matrix of a neural network layer by $W \in \mathbb{R}^{m \times n}$, and the 
corresponding gradient by $G \in \mathbb{R}^{m \times n}$. At a given time step $t$, these are denoted 
as $W_t$ and $G_t$, respectively. For a batch of inputs at time $t$, denoted 
by $B_t$, the loss and its gradient evaluated at $W_t$ are represented as 
$\phi_{B_t}(W_t)$ and $\nabla_W \phi_{B_t}(W_t)$, respectively.

Adagrad \citep{adagrad} is an online learning second-order algorithm that maintains a preconditioner $H \in \mathbb{R}^{mn \times mn}$. If the vectorized gradient at time $t$ is denoted by $g_t$ (i.e., $g_t = \text{vec}(G_t) \in \mathbb{R}^{mn}$), then the update of the preconditioner and the vectorized weights $w_t \in \mathbb{R}^{mn}$ with learning rate $\eta$ is given by

\[ H_t = H_{t-1} + g_t g_t^\top; \quad w_t = w_{t-1} - \eta H_t^{-1/2} g_t \]

Adam \citep{adam}, a widely used first-order optimization algorithm in 
deep learning is a diagonal approximation of Adagrad. It maintains an exponential moving average of the gradients 
$G_t$ (denoted as $M_t$) and of element-wise squared gradients $G_t^2$ 
(denoted as $V_t$) for a given weight matrix $W$. Its update rule with 
learning rate $\eta$ is given by
\[
W_t \leftarrow W_{t-1} - \eta \frac{M_t}{\sqrt{V_t}},
\]
where the division is performed element-wise.

Adafactor \citep{adafactor, zhai}, a variant of Adam, replaces $V_t$ with 
its best rank-1 approximation $V_t'$ to reduce memory usage. While the 
original Adafactor paper \citep{adafactor} proposed additional modifications, 
such as changes to the learning rate schedule, we focus on the version of 
Adafactor proposed in recent works \citep{zhai, zhaoscience}, whose update 
with learning rate $\eta$ is given by
\[
W_t \leftarrow W_{t-1} - \eta \frac{M_t}{\sqrt{V_t'}}.
\]

Shampoo \citep{shampoo} is a second-order optimization algorithm that approximates Adagrad and maintains two preconditioners, $L_t \in \mathbb{R}^{m \times m}$ and $R_t \in \mathbb{R}^{n \times n}$, for a given weight matrix $W \in 
\mathbb{R}^{m \times n}$. The updates for the preconditioners and the 
weights with learning rate $\eta$ are as follows:
\[
L_t \leftarrow L_{t-1} + G_tG_t^T; \quad R_t \leftarrow R_{t-1} + G_t^TG_t; 
\quad W_t \leftarrow W_{t-1} - \eta L_t^{-1/4} G_t R_t^{-1/4}.
\]

In practice, Shampoo is implemented with several other modifications such as layerwise learning rate grafting and exponents other than $-1/4$. We will use the DistributedShampoo~\citep{distributedshampoo} implementation which has these variations available as hyperparameters.

\section{Related Work}
\label{sec:related}
We begin by discussing works that are closely related, including \citet{ekfac, 
anil2020scalable} and \citet{galore}. Subsequently, we  cover extended related works.

\textbf{KFAC}~\citep{martens15} is a well-known second-order optimization algorithm designed 
for neural networks. E-KFAC \citep{ekfac} builds upon KFAC in a manner analogous to 
our extension of Shampoo, introducing a diagonal preconditioner that is updated 
between KFAC inversion steps. However, E-KFAC's algorithm is not identical to running 
Adam in KFAC’s eigenbasis, as the diagonal preconditioner is not Adam.

\citet{anil2020scalable} introduced several algorithmic and numerical improvements to 
develop a practical and scalable version of Shampoo \citep{shampoo}. Notably, they 
empirically found that using an exponent of $1/2$ outperforms the original exponent 
of $1/4$ in Shampoo. Of particular interest to our work is Appendix B of 
\citet{anil2020scalable}, where, inspired by E-KFAC, they describe an algorithm that 
is essentially equivalent to SOAP for 2D layers. However, no experiments were 
provided, and the authors claimed that unpublished experiments showed no empirical 
improvement over Shampoo. This discrepancy between our findings may be due to some of the implementation details of SOAP.

\textbf{GaLore}~\citep{galore} was recently proposed as a method to reduce Adam's memory footprint by 
maintaining momentum in a low-rank subspace derived from the singular value 
decomposition (SVD) of the gradients. Their algorithm’s full-rank version bears 
similarity to ours, with some notable distinctions. Firstly, their projection subspace is 
determined by the SVD of the current gradient, while we maintain an exponential moving 
average of $GG^T$ and $G^TG$. Secondly, we retain momentum in the original space and 
project it onto the preconditioned space, whereas they maintain it in the preconditioned 
space and do not rotate it each time the preconditioned space is updated. In~\Cref{app:galore}, we study GaLore's performance and find that our modifications are necessary for improving upon Shampoo. Moreover, their method 
only projects one side of a layer using the eigenbasis while using the identity basis on the other side. We examine the impact of one-sided 
projection for SOAP in~\Cref{sec:one-sided}.

\textbf{Diagonal Preconditioning based Optimizers:} Other than AdamW, there are other optimizers which involve diagonal preconditoning such as Lion~\citep{lion}, Sophia~\citep{liu2024sophia}, and Adafactor~\citep{adafactor}. Recent works of~\citet{KaddourKNMK23,zhaoscience} showed that these optimizers perform comparably to AdamW for LLM pretraining but do not surpass it. This suggests the need to explore non-diagonal preconditioners. We discuss prior works on non-diagonal preconditioners below.

\textbf{Second-Order Optimization:} Research on second-order optimization in deep 
learning is generally divided into two categories: Hessian-free methods and methods 
that estimate the Hessian.

\textbf{Hessian-Free Methods:} Hessian-free approaches \citep{martens10, martens15} 
optimize without explicitly computing the Hessian matrix, instead employing iterative 
techniques to approximate the Newton step. Other recent works \citep{li18psgd, 
li2024stochastichessianfittingslie, pooladzandi2024curvatureinformed} have focused on 
designing iterative preconditioners to improve the convergence specifically for stochastic
optimization algorithms.

\textbf{Hessian Estimation Methods:} These methods maintain an efficient 
approximation of the Hessian for neural networks. KFAC \citep{martens15} and Shampoo 
\citep{shampoo} are two widely recognized methods in this area.

KFAC \citep{martens15} was one of the first approaches to go beyond diagonal 
preconditioners in neural networks, demonstrating that a layer-wise Kronecker-factored 
preconditioner approximates the layer-wise Hessian in multi-layer perceptrons (MLPs). 
Subsequent works \citep{martens2018kroneckerrecurrent, kazuki19} extended KFAC to other 
architectures. Recent research \citep{ekfac, Gao_Liu_Huang_Wang_Wang_Xu_Yu_2021} has 
further improved trace and diagonal estimates for KFAC. Efforts to scale up KFAC 
\citep{ba2017distributed, Puiu22, puiu2023brandnewkfacsspeeding, eschenhagen2023kroneckerfactored} 
have focused on making the inversion step more efficient or enhancing distributed 
implementations.

Shampoo \citep{shampoo}, another second-order optimization algorithm, is motivated by 
the online learning algorithm Adagrad \citep{JMLR:v12:duchi11a}. Shampoo also employs 
a layer-wise Kronecker-factored preconditioner. A recent distributed implementation of 
Shampoo \citep{distributedshampoo} won an optimization efficiency benchmark 
\citep{dahl2023benchmarking}, highlighting the practical utility of second-order methods 
in deep learning. Few recent works \citep{caspr, whyshampoo} have provided theoretical advancements on top of Shampoo. Other works \citep{anil2020scalable, peirson2022fishy, wulin2024, 
4bitshampoo} have proposed various strategies to improve Shampoo’s scalability. We defer 
a comparison of SOAP with these methods to future work.

\section{Algorithm}
\label{sec:alg}

\subsection{Theory}

We begin by describing an equivalence between Shampoo and running Adafactor in the eigenbasis of the Shampoo preconditioner. For simplicity we omit momentum but the equivalence also holds with momentum. For this equivalence we  use Shampoo with the following modifications from the original Shampoo optimizer~\citep{shampoo}:
\begin{enumerate}
    \item We  use power $1/2$ instead of power $1/4$. This was already recommended in practical implementations~\citep{anil2020scalable,distributedshampoo} and a theoretical connection between optimal Kronecker approximation of Adagrad \citep{adagrad} preconditioner and Shampoo with power $1/2$ was established in~\cite{whyshampoo}. 
    \item  We also use the scalar correction to per layer learning rates described in \citet{yi21,whyshampoo}.
    \item Instead of the running average of $L$ and $R$ across time steps, we use dataset averages.
\end{enumerate}

With these changes in place (first occurrence of these changes is highlighted in \red{red} in the algorithm below) we formally define the two algorithms whose equivalence we  show in Algorithms \ref{alg:shampoo_ideal} and \ref{alg:shampoo_adafactor}.

\begin{algorithm}[H]

\begin{algorithmic}[1]
    \STATE Sample batch $B_t$.
    \STATE $G_{t} \in \mathbb{R}^{m \times n} \gets -\nabla_W \phi_{B_t}(W_t)$
    \STATE $L \gets \red{\E_{B}}[G_{B}G_{B}^T]$ \COMMENT{Where the expectation is over a random batch $B$.}
    \STATE $R \gets \red{\E_{B}}[G_{B}^TG_{B}]$
    \STATE $\hat{H} \gets L \otimes R / \red{\text{Trace}(L)}$
    \STATE $W_t \gets W_{t-1} - \eta \hat{H}^{\red{-1/2}}G_{t} = W_{t-1} - \eta L^{-1/2}G_{t}R^{-1/2}/ \text{Trace}(L)^{-1/2}$
\end{algorithmic}
\caption{Single step of idealized Shampoo with power $1/2$.}
\label{alg:shampoo_ideal}
\end{algorithm}

\begin{algorithm}[H]

\begin{algorithmic}[1]
    \STATE Sample batch $B_t$.
    \STATE $G_{t} \in \mathbb{R}^{m \times n} \gets -\nabla_W \phi_{B_t}(W_t)$
    \STATE $L \gets \E_{B}[G_{B}G_{B}^T]$
    \STATE $R \gets \E_{B}[G_{B}^TG_{B}]$
    \STATE $Q_L \gets \texttt{Eigenvectors}(L)$
    \STATE $Q_R \gets \texttt{Eigenvectors}(R)$
    \STATE $G'_{t} \gets Q_L^T G_{t} Q_R$
    \STATE \COMMENT{Idealized version of code for Adafactor taking $G'_{t}$ to be the gradient}
    \STATE $G'_{B} \gets Q_L^T G_{B} Q_R$ 
    \STATE $A = \E_{B}[G'_{B} \odot G'_{B}] \mathbf{1}_m$ where $G'_{B} = Q_L^T G_{B} Q_R$
    \STATE $C = \mathbf{1}_n^\top \E_{B}[G'_{B} \odot G'_{B}]$
    \STATE $\hat{V}_t = \frac{A C^T}{\mathbf{1}_n^\top A}$ \COMMENT{Elementwise division}
    \STATE $G''_{t} \gets \frac{G'_t}{\sqrt{\hat{V}_t} + \epsilon}$ \COMMENT{Elementwise division and square root}
    \STATE $G'''_{t} \gets Q_L G''_{t} Q_R^T$ \COMMENT{Projecting back to original space}
    \STATE $W_t \gets W_{t-1} - \eta G'''_{t}$
\end{algorithmic}
\caption{Single step of idealized Adafactor in Shampoo's eigenspace.}
\label{alg:shampoo_adafactor}
\end{algorithm}

\begin{claim}
\label{claim:eq}
\Cref{alg:shampoo_ideal,alg:shampoo_adafactor} are equivalent. 
\end{claim}
\begin{proof}

Consider $G_t$ in the basis created after rotating by $Q_L, Q_R$ i.e. $G'_t = Q_L^T G_t Q_R$. Let the eigenvalues of $\E_{B}[G_{B}G_{B}^T]$ and $\E_{B}[G_{B}^TG_{B}]$  be given by $\lambda_1,...,\lambda_m$ and $\mu_1,...,\mu_n$ respectively. Algorithm 1 scales the $i, j$ coordinate by $(\lambda_i \mu_j / (\sum_i \lambda_i))^{-1/2}$, while Algorithm 2 scales them by  $(A_i C_j / (\sum_i A_i))^{-1/2}$. We  now show that $A_i = \lambda_i$, an analogous argument shows $C_j = \mu_j$.

\begin{align*}
    A_i &= e_i^T \E_{B}[G'_{B} \odot G'_{B}] \mathbf{1}_m \\
    &= \E_{B}[ \sum_j (G'_{B})_{i, j}^2] \\
    &= \E_{B}[ \sum_j (u_i^T(G_{B})v_j)^2] \quad \quad (\text{Using definition of } G') \\
    &= \E_{B}[  ||u_i^T(G_{B})||^2] \quad \quad \quad \quad (v_j \text{ form a basis})\\
    &= \E_{B}[ u_i^TG_{B}G_{B}^Tu_i] \\
    &= \lambda_i \quad \quad \quad \quad \quad \quad \quad \quad \quad \quad (\text{By definition of } \lambda_i \text{ and } u_i)
\end{align*}

\end{proof}

While these two algorithms are equivalent in their idealized forms, practical considerations reveal some differences. Firstly, the algorithms differ when using running averages instead of dataset averages.
Secondly, and more significantly in practice, we do not invert or compute the eigenvector decomposition of $L$ and $R$ at every step. This means that the ``adaptivity'' of learning rates in Shampoo is limited\footnote{We note that practical implementations of Shampoo use grafting which allows for learning rate adaptivity at every step, but this adaptivity is restricted to a single scalar per layer.} to the updates of $L$ and $R$. In contrast, with Adafactor in Shampoo's eigenspace, the second moment estimates (i.e., $A$ and $C$ in Algorithm \ref{alg:shampoo_adafactor}) can be updated at every step as they are computationally inexpensive. Additionally, instead of using Adafactor, we can opt\footnote{Though using AdamW over Adafactor only gives very small improvements in performance, see~\Cref{fig:factor} and~\Cref{app:adafactor}. We also note that one can use any other diagonal preconditioner based optimizer in place of Adam, such as Lion~\citep{lion}, Sophia~\citep{liu2024sophia} or Schedule-Free AdamW~\citep{schedulefree}.} for Adam, which offers more generality. Combining these insights leads to~\Cref{alg:SOAP} which can be interpreted as running Adam in Shampoo's eigenspace.

\begin{algorithm}[t]
	
	\begin{algorithmic}[1]
		\STATE Sample batch $B_t$.
		\STATE $G \in \mathbb{R}^{m \times n} \gets -\nabla_W \phi_{B_t}(W_t)$
		\STATE $G' \gets Q_L^T G Q_R$
		\STATE $M \gets \beta_1 M + (1-\beta_1) G$
		\STATE $M' \gets Q_L^T M Q_R$
		\STATE \COMMENT{Now we  ``run'' Adam on $G'$}
		\STATE $V \gets \beta_2 V + (1-\beta_2) (G' \odot G')$ \COMMENT{Elementwise multiplication}
		\STATE $N' \gets \frac{M'}{\sqrt{\hat{V}_t} + \epsilon}$ \COMMENT{Elementwise division and square root}
		\STATE \COMMENT{Now that we have preconditioned by Adam in the rotated space, we go back to the original space.}
		\STATE $N \gets Q_L N' Q_R^T$ 
		\STATE $W \gets W - \eta N$
		\STATE \COMMENT{End of gradient step, we now update $L$ and $R$ and possibly also $Q_L$ and $Q_R$. }
		\STATE $L \gets \beta_2  L + (1-\beta_2) GG^T$
		\STATE $R \gets \beta_2  R + (1-\beta_2) G^TG$
		\IF{t \% f == 0}
		\STATE $Q_L \gets \texttt{Eigenvectors}(L, Q_L)$
		\STATE $Q_R \gets \texttt{Eigenvectors}(R, Q_R)$
		\ENDIF
	\end{algorithmic}
	\caption{Single step of SOAP for a $m \times n$ layer. Per layer, we  maintain four matrices: $L \in \mathbb{R}^{m \times m}, R \in \mathbb{R}^{n \times n}$ and $V, M \in \mathbb{R}^{m \times n}$. For simplicity we  ignore the initialization and other boundary effects such as bias correction. Hyperparameters: Learning rate $\eta$, $\text{betas} = (\beta_1, \beta_2)$, epsilon $\eps$, and preconditioning frequency $f$.\\ An implementation of SOAP is available at \url{https://github.com/nikhilvyas/SOAP/tree/main}.}%
	\label{alg:SOAP}
\end{algorithm}

\begin{algorithm}[t]
	
	\begin{algorithmic}[1]
		\STATE $S \gets PQ$
            \STATE $Q \gets \texttt{QR}(S)$
	\end{algorithmic}
	\caption{\texttt{Eigenvectors} function, implemented using power iteration and QR decomposition. Inputs: PSD matrix $P$ and estimate of eigenvectors $Q$. If the estimate was exact we would have $P = Q D Q^T$ where $D$ is the diagonal matrix with eigenvalues.}
	\label{alg:eigenvectors}
\end{algorithm}

We now describe some additional implementation details:
\begin{enumerate}
	\item ~\Cref{alg:SOAP} describes the behavior of the algorithm for 2D layers. Following~\citet{galore}, for 1D layers we run standard AdamW. This reduces the overhead as compared to standard implementations of Shampoo which solve an eigenvector problem for 1D layers too. 
	\item Following~\citet{4bitshampoo}, we compute eigenvectors of $L$ (and $R$) using one step of power method (\Cref{alg:eigenvectors}). This requires doing one matrix multiplication followed by QR decomposition. QR decomposition is faster~\citep{eigh-vs-qr} than standard eigenvector decomposition in PyTorch. For the first iteration, eigenvectors are initialized by doing a standard eigenvector decomposition.  %
	\item For layers with huge dimensions such as the first and last layer in language modeling transformers, maintaining the eigenvectors would be space and time prohibitive. For such dimensions we fix the rotation matrix ($Q_L$ or $Q_R$) to be identity. Note that if we fix both $Q_L$ and $Q_R$ to be identity for a 2D layer, we would recover Adam.
        \item ~\Cref{alg:SOAP} omits bias correction and weight decay for simplicity, but these are used in the actual implementation, identical to their use in AdamW.
\end{enumerate}

The main focus of the next sections will be to explore the empirical performance of this algorithm and its variations. n~\Cref{app:adafactor,app:theory-overhead} we discuss the the space and time complexity of SOAP and how it can be improved.

\section{Experimental Methodology}
\label{sec:methods}

\textbf{Hyperparameter tuning:} We begin with hyperparameter values suggested by prior research for both AdamW and Distributed Shampoo (e.g., $\beta_2 = 0.95$). Initially, we conduct a learning rate sweep to determine the optimal learning rate for each optimizer. Once the optimal learning rate is identified, we perform two-dimensional sweeps for each of the remaining hyperparameters, where we vary the selected hyperparameter alongside the learning rate. The purpose of these sweeps is to demonstrate that our default hyperparameter settings are near-optimal, disregarding potential interactions between two non-learning-rate hyperparameters. A detailed discussion of the hyperparameter sweeps is provided in~\Cref{app:setup}.
\begin{figure}[ht]
    \centering
    \includegraphics[width=1\linewidth]{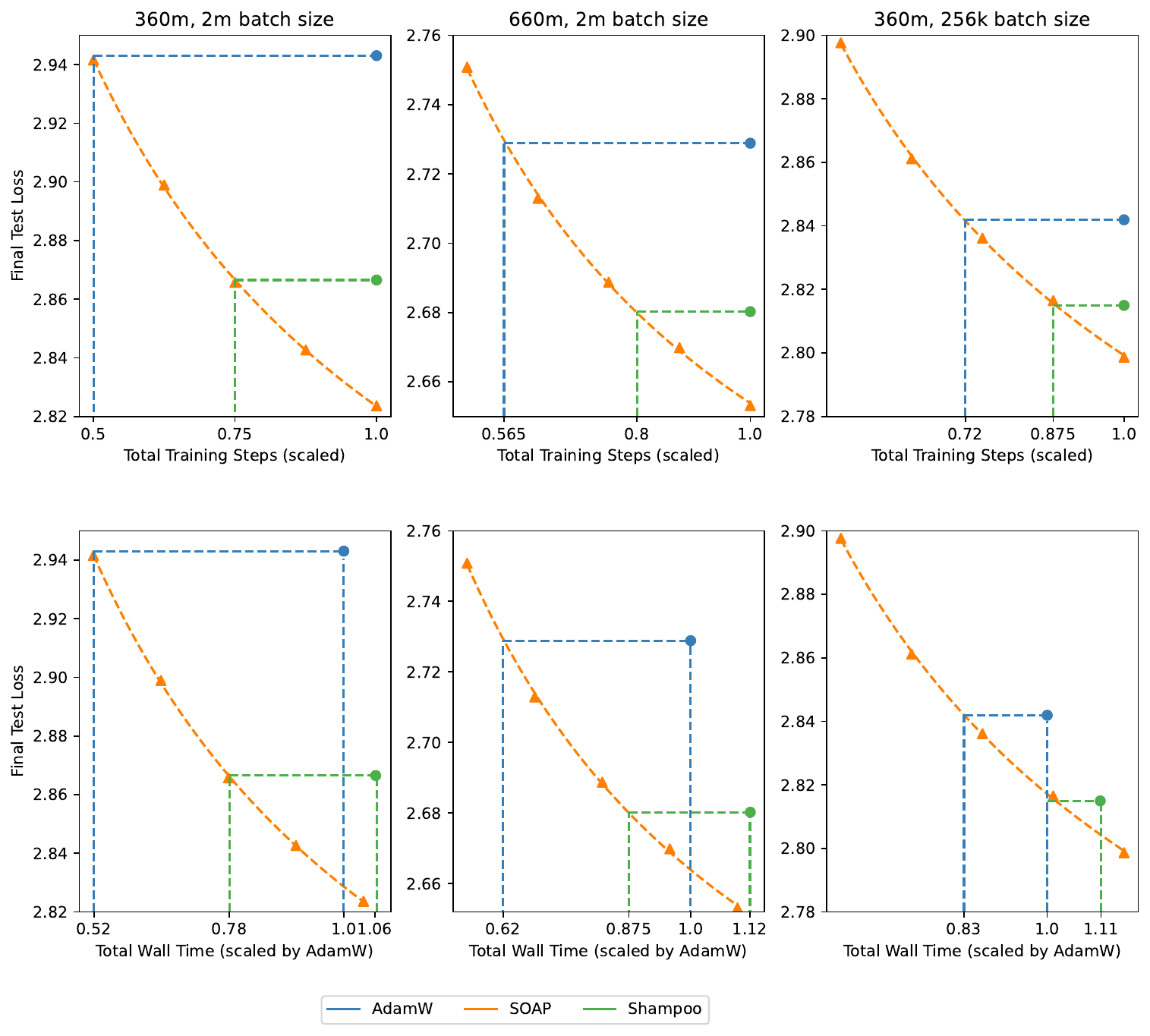}
    \caption{Precise efficiency benefits of SOAP over AdamW and Shampoo for 360m (at 256k and 2m batch size) and 660m (at 2m batch size) model. For the precise methodology, refer to~\Cref{sec:methods}.}
    \label{fig:prec_overhead}
\end{figure}

\textbf{Throughput Measurement:} We evaluate the throughput of each optimizer by measuring the number of tokens processed per second. At present, we perform these measurements on a single H100 GPU and utilize gradient accumulation to accommodate large batch sizes. While this approach may seem to disadvantage AdamW--- as the overhead of Shampoo/SOAP is compared against multiple gradient accumulation steps--- it is important to note that the overhead of Shampoo/SOAP can be amortized across layers by distributing the updates across multiple GPUs. This technique is employed in the distributed implementation of Shampoo \citep{distributedshampoo}. A comprehensive comparison of distributed implementations of these algorithms is left to future work.

\textbf{Efficiency Benefits:} Simply running SOAP for the same duration as Shampoo and AdamW cannot be directly used to calculate the efficiency benefit (in terms of training steps or wall-clock time) of using SOAP since we use a cosine schedule. Therefore, we run SOAP on $.5, .625, .75$ and $.875$ fraction of the training data and fit a scaling law of the form $a + bN^{-\beta}$ through the final losses obtained, where $N$ represents the number of training points and $a,b,\beta$ are the parameters of the fit. We show these points and the corresponding scaling laws obtained in~\Cref{fig:prec_overhead}. This scaling law is then used to calculate the efficiency benefit in terms of training steps and wallclock time as shown in~\Cref{fig:prec_overhead}. Here, the horizontal lines represent the final losses of AdamW and Shampoo. 

\section{Language Modeling Experiments}
\label{sec:language}

In this section we focus on empirically comparing AdamW, DistributedShampoo, and SOAP on language modeling tasks. 
\begin{figure}[ht]
    \centering
    \includegraphics[width=1\linewidth]{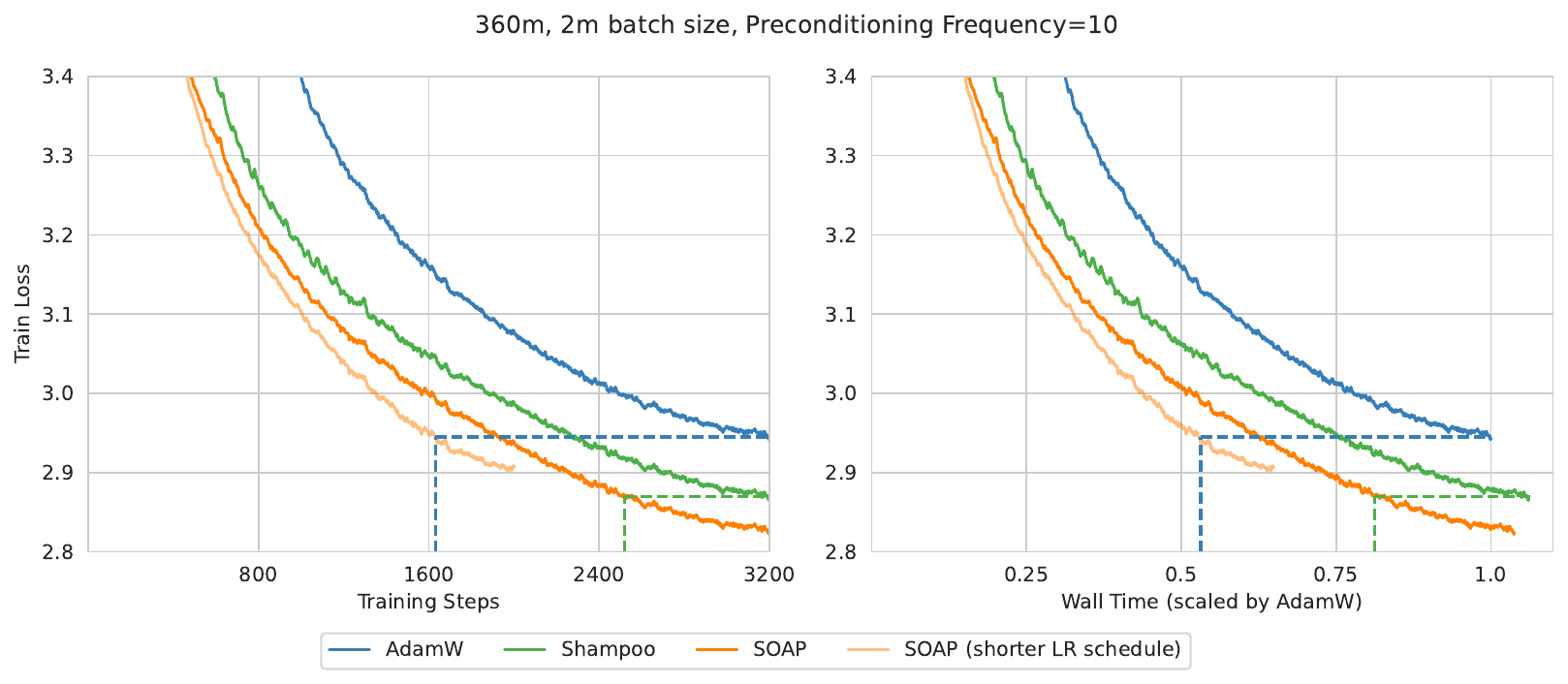}
    \caption{Comparing performance of tuned runs for AdamW, Shampoo (using DistributedShampoo~\citep{distributedshampoo} implementation) and SOAP. Shampoo and SOAP use preconditioning frequency of 10. We observe a $\geq 40\%$ reduction in the number of iterations and a $\geq 35\%$ reduction in wall clock time compared to AdamW, and approximately a $20\%$ reduction in both metrics compared to Shampoo. See ~\Cref{fig:main} for 660m results, ~\Cref{sec:freq,sec:critical} for ablations of preconditioning frequency and batch size respectively, and~\Cref{sec:methods} for detailed calculation of efficiency improvement and experimental methodology.}
    \label{fig:main360}
\end{figure}
\subsection{Measuring Efficiency Benefits}

In~\Cref{fig:main} (left and middle) and~\Cref{fig:main360} we show train loss curves for AdamW, Shampoo, and SOAP on 360m and 660m models with 2m token batch size and ``chinchilla-optimal'' i.e. 20x model size number of tokens. In these plots we observe that SOAP outperforms the other two optimizers. To directly calculate the efficiency benefit of SOAP, we also run SOAP with cosine decay for a shorter lr schedule, as shown in~\Cref{fig:main,fig:main360}. This allows us to approximate the following efficiency benefits (when batch size is set to 2m and preconditioning frequency to 10): $\geq 40\%$ reduction in the number of iterations and $\geq 35\%$ reduction in wall clock time compared to AdamW; $\approx 20\%$ reduction in iterations and wall clock time as compared to Shampoo. Precise efficiency benefit calculations are presented in ~\Cref{fig:prec_overhead}(left and middle). In~\Cref{sec:long} we show that efficiency benefits of SOAP over AdamW are maintained for longer duration runs where \#tokens = 100 $\times$ model size.

\subsection{Effect of Frequency of Finding Eigenvectors/Inverse}
\label{sec:freq}

In~\Cref{fig:main} (right), we compare SOAP and Shampoo with respect to preconditioning frequency. We observe the following: 
\begin{itemize}
    \item For all frequencies we tried from 1 to 100, both optimizers outperform AdamW.
    \item At frequency 1, SOAP and Shampoo are quite close in performance.
    \item At higher frequencies, the performance of  both SOAP and Shampoo degrades but SOAP's performance degrades significantly slower than Shampoo's.
\end{itemize}

\begin{figure}[!h]
    \centering
    \includegraphics[width=1.0\linewidth]{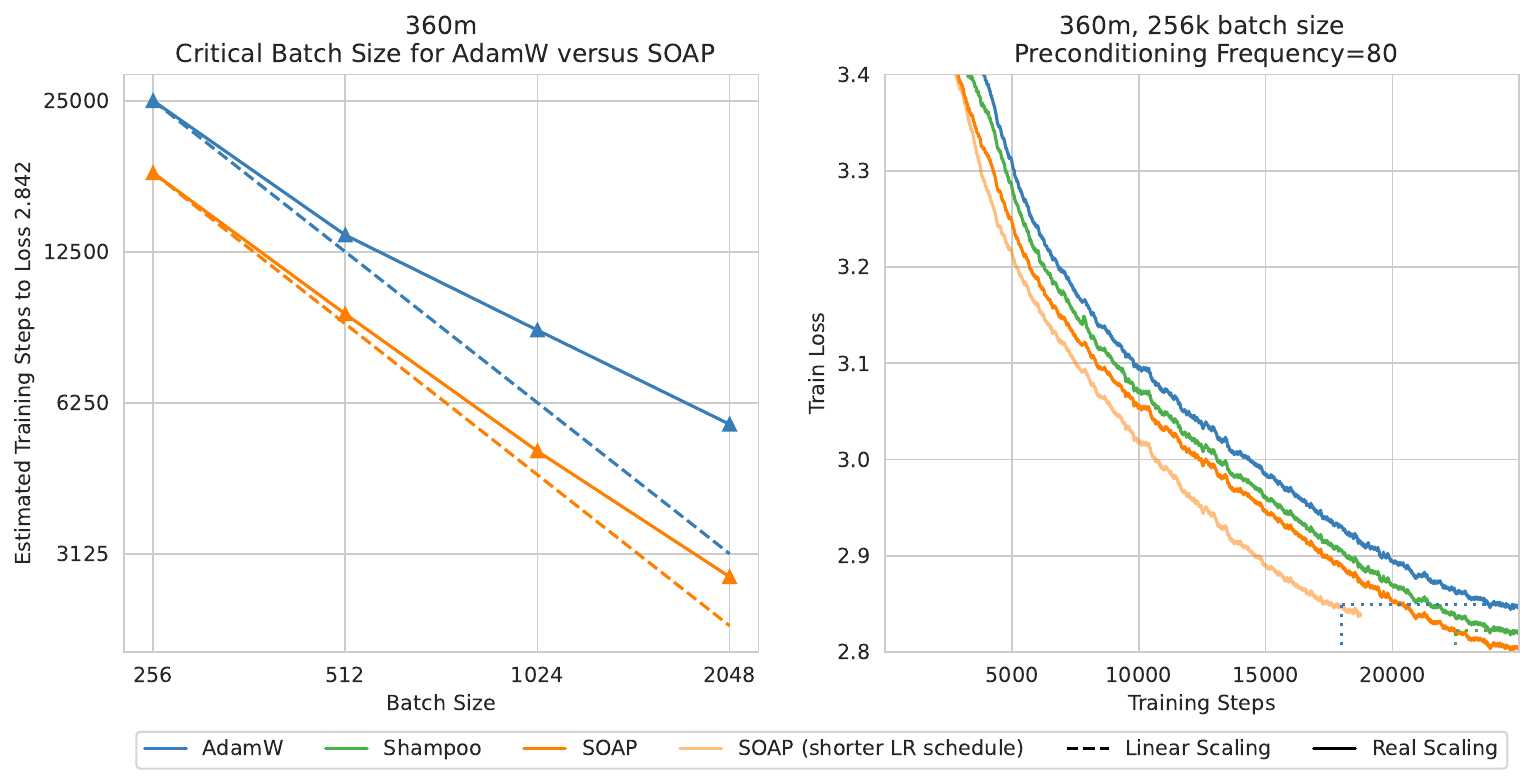}
    \caption{(left) Comparing the critical batch size of AdamW vs SOAP. We can see that SOAP improves the critical batch size, by being much closer to the ideal linear scaling with batch size as compared to AdamW. (right) Comparing performance of tuned runs for AdamW, Shampoo (using DistributedShampoo~\citep{distributedshampoo} implementation) and SOAP for token batch size of 256k. Shampoo and SOAP use preconditioning frequency of 80. We observe a $\geq 25\%$ reduction in the number of iterations compared to AdamW, and approximately a 10\% reduction compared to Shampoo. See~\Cref{fig:prec_overhead} (right) for wall-clock time improvement and ~\Cref{sec:methods} for detailed calculation of efficiency improvement.}
    \label{fig:bsz}
\end{figure}

\subsection{SOAP Improves the Critical Batch Size}
\label{sec:critical}

When scaling up batch sizes, the ideal outcome is that doubling the batch size results in halving the number of training steps needed to achieve the same performance. The batch size at which this ideal scaling starts to break down is referred to by~\citet{critical} as the \textit{critical batch size}. As models and datasets grow larger, it becomes increasingly important to develop optimizers that support larger critical batch sizes, thereby reducing the serial runtime of a training run. In this subsection, we compare the critical batch sizes of AdamW and SOAP. Relative to our baseline setup of a 2 million batch size, when we decrease the batch size by a factor of $k$, we increase the preconditioning frequency by the same factor. This ensures that the FLOPS and wall clock multiplicative overhead for the eigenvector decomposition steps remains consistent with the 2 million batch size setting.

We start by training a 360 million parameter model with a batch size of 256k for a "Chinchilla-optimal" number of tokens (20 times the model size) using AdamW, achieving a loss of 2.842. This value is set as the target loss for our comparisons. In~\Cref{fig:bsz} (left), we show the number of steps AdamW and SOAP require to reach this target loss as we vary the batch size. SOAP consistently requires fewer steps across all batch sizes, with the multiplicative benefits becoming more pronounced at larger batch sizes. Additionally, we compare these results to the ideal scenario (dashed line) of linear scaling, where doubling the batch size halves the number of steps. SOAP more closely follows the linear scaling trend compared to AdamW, indicating that it has a higher critical batch size in this setup.

In~\Cref{fig:bsz} (right), we present the optimal runs for each optimizer (including Shampoo) at the smallest batch size we consider: 256k. SOAP outperforms both Shampoo and AdamW, reducing the number of iterations by 25\% compared to AdamW, and by approximately 10\% compared to Shampoo. Furthermore, in~\Cref{fig:prec_overhead} (right, bottom), we demonstrate that SOAP also achieves a wall-clock time improvement of $\geq 15\%$ over AdamW and around 10\% over Shampoo. We note that these results are a preliminary analysis for smaller batch size runs. Our approach of keeping the product of batch size and preconditioning frequency constant may not be optimal, and a better trade-off could likely be found. Furthermore, SOAP’s overhead could potentially be reduced by performing $L$ and $R$ updates in lower precision (instead of fp32). Finally, the diminished efficiency gains of second-order methods at smaller batch sizes are consistent with prior findings~\citep{nqm,ishikawa2024when}.

\subsection{Scaling to Larger Token Counts}
\label{sec:long}

Thus far, our focus has been on Chinchilla-optimal token counts for a given model size. However, in many practical scenarios, models are trained on significantly larger token budgets to optimize inference costs and downstream performance. In Figure~\ref{fig:long}, we demonstrate that SOAP maintains its advantage Adam even in extended training runs.

\begin{figure}[!h]
    \centering
    \includegraphics[width=1.0\linewidth]{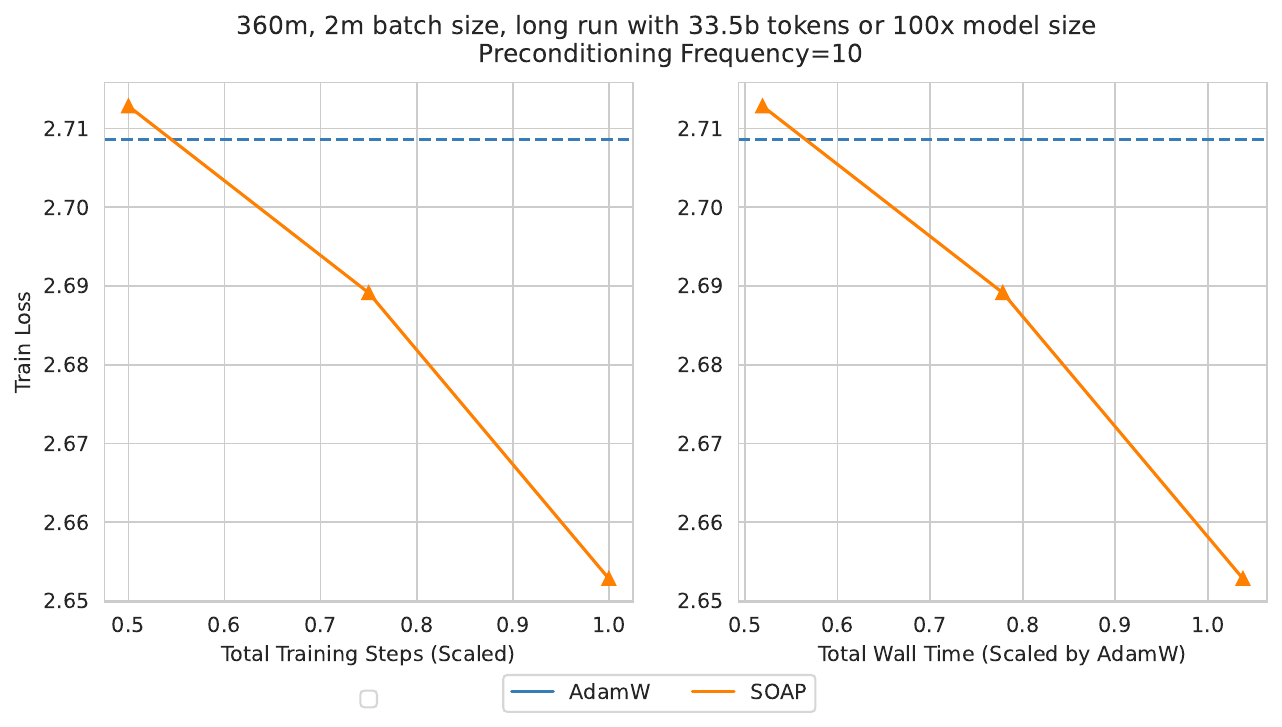}
    \caption{Performance comparison of SOAP and Adam for longer duration training runs.}
    \label{fig:long}
\end{figure}

\section{Further Efficiency Improvements}
\label{sec:efficiency}

In this section, we discuss space and time complexity of SOAP and provide an overview of potential avenues for further space and compute efficiency improvements in SOAP.

\subsection{One Sided Eigenbasis}

\label{sec:one-sided}
As described in~\Cref{sec:related}, \citet{galore} have an algorithm similar to ours. One of the differences is that they only project the smaller side of the layer using the eigenbasis while using identity as the rotation matrix for the larger side i.e. if $m < n$ we set $Q_R = I_n$ in~\Cref{alg:SOAP} and if $m > n$ we set $Q_L = I_m$. Doing this leads to a reduction in space usage as well as reduction of optimizer time overhead, which is discussed in~\Cref{app:impr-space,app:impr-time}.

In~\Cref{fig:factor}, it is evident that the one-sided projection results in slightly reduced performance compared to the original SOAP optimizer. However, it still performs on par with, or marginally better than, Shampoo, while maintaining greater computational efficiency. Further investigation into the potential for these variants to surpass the computational efficiency of original SOAP optimizer is left for future work.

\begin{figure}[!h]
    \centering
    \includegraphics[width=1\linewidth]{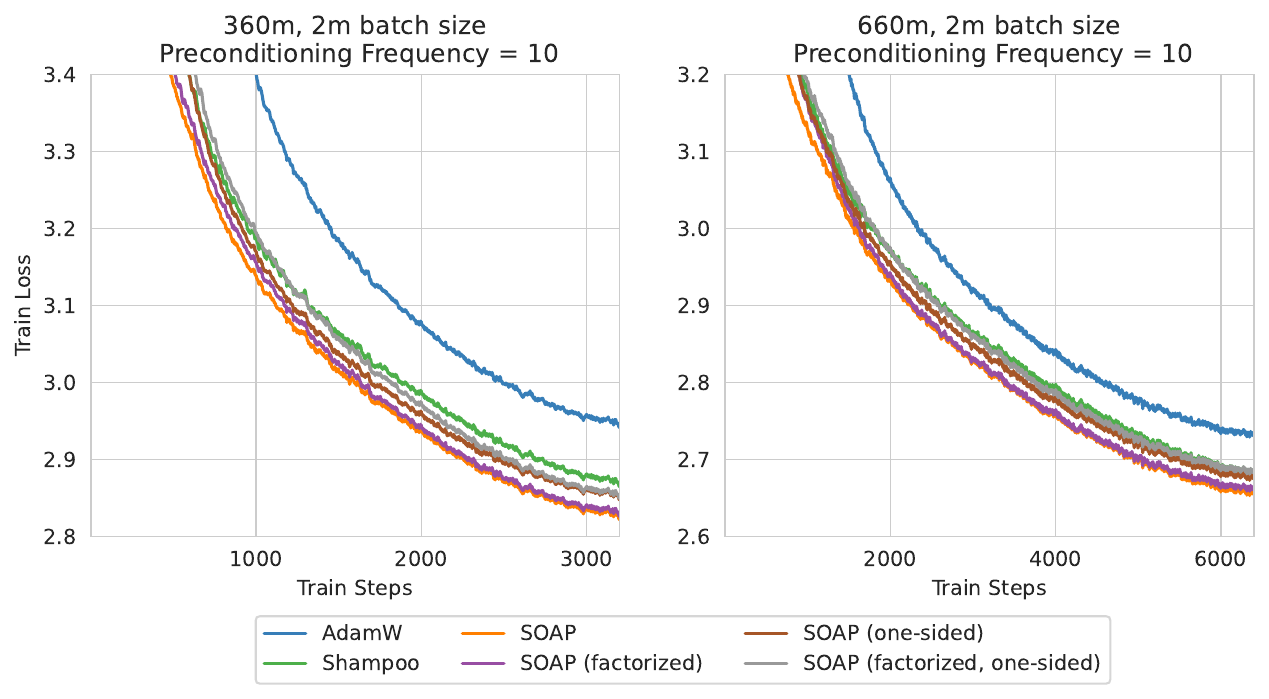}
    \caption{Performance of variants of SOAP which improve space usage or time overhead. 1. SOAP (factorized): Uses Adafactor instead of Adam in Shampoo's eigenbasis and 2. SOAP (one-sided): Uses $Q = I$ (i.e. no rotation) on the large side of weight matrix and 3. SOAP~(factorized,~one-sided): Combines both of these changes. We observe that while using Adafactor instead of Adam causes a negligible increase in loss, using the one-sided variant causes a larger increase. However, the one-sided variant also has much larger reduction in time and space overhead. For computational benefits of these variants see~\Cref{app:adafactor,app:theory-overhead}.}
    \label{fig:factor}
\end{figure}

\subsection{Space usage of SOAP}
\label{app:adafactor}

For a $m \times n$ matrix where $m > n$ we require $$2m^2\text{ (for $L, Q_L$)}+2n^2\text{ (for $R, Q_R$)}+3mn\text{ (for gradient, $M, V$)}$$ space usage\footnote{One $mn$ is for storing the gradients, this can be avoided (as long as there is no gradient accumulation) by applying gradients along with backprop~\citep{lomo} but this is not implemented by default in standard deep learning frameworks such as PyTorch. Hence we will include this term in all of our calculations.} (beyond weights and activations), specifically for $L, Q_L, R, Q_R, \text{momentum }(M)$, AdamW's second order estimate ($V$), and the gradient. This is the same space usage as DistributedShampoo while AdamW uses $3mn$.

\subsubsection{Improving space usage of SOAP}
\label{app:impr-space}

The most direct way to reduce memory is using low precision to store the $L, R, Q_L, Q_R, V$ matrices, which is done by~\citet{8bitadam,4bitshampoo}. Orthogonal to the low precision approaches, there are two algorithmic approaches to improving the space usage of SOAP: 
\begin{itemize}
    \item Using Adafactor instead of Adam as the diagonal preconditioner after rotating by $Q_L$ and $Q_R$. This reduces the space usage by $mn$.
    \item Using one sided version of SOAP (\Cref{sec:one-sided}). This reduces space usage from $2m^2+2n^2+3mn$ to $2 \min(m,n)^2 + 3mn$.
    \item Combining these approaches yields space usage of $2 \min(m,n)^2 + 2mn$.
\end{itemize} 
For standard transformer architectures the last variant which combines the two approaches would yield less space usage overall compared to AdamW (which uses $3mn$).

We try these approaches in Figure~\ref{fig:factor}. We observe that using Adafactor instead of AdamW yields very small reductions in performance while using one-sided preconditioner results in larger reductions. Nonetheless even after combining these two approaches the resulting optimizer outperforms AdamW while having a smaller space requirement than AdamW. Regarding space usage we also note that Adafactor (with momentum added back) itself utilizes only $2mn$ space usage and has been shown to perform comparable to AdamW for ViT training~\citep{zhai} and for language model training~\citep{zhaoscience}. Further space reduction beyond Adafactor has been studied in the Adalomo~\citep{adalomo}, GaLore~\citep{galore}, and AdaMeM~\citep{adamem} papers.

\subsection{Time Overhead of SOAP}
\label{app:theory-overhead}
There are two types of overhead of Shampoo and SOAP over AdamW: the overhead per step and the overhead when changing the preconditioner (or for SOAP, the preconditioner's eigenbasis). Let us first analyze the first one. For SOAP per step for a layer of size $m \times n$ we have an overhead of $$m^3\text{ (updating $L$)}+n^3\text{ (updating $R$)}+(2m^2n+2mn^2)\text{ (projecting and projecting back on both sides)}.$$

We note that this is more than the overhead of Shampoo which is $m^3+n^3+m^2n+n^2m$. This can be observed in~\Cref{fig:prec_overhead} (bottom, right) but not in the other figures since there the second type of overhead is the dominant term.

The second type of overhead is due to changing the preconditioner for Shampoo (or for SOAP, preconditioner's eigenbasis i.e. $Q_L$ and $Q_R$). The DistributedShampoo \citep{distributedshampoo} implementation of Shampoo uses a direct call to $\texttt{torch.linalg.eigh}$ for this. Following~\citet{4bitshampoo} we use~\Cref{alg:eigenvectors} which uses power iteration based approach which calls $\texttt{torch.linalg.qr}$. We note that $\texttt{torch.linalg.qr}$ is faster than $\texttt{torch.linalg.eigh}$~\citep{eigh-vs-qr}. In Figure~\ref{fig:linalg-runtime} (right) we see that using power iteration based approach ($\texttt{torch.linalg.qr}$) performs as well as fresh eigenvector decomposition ($\texttt{torch.linalg.eigh}$).

\label{app:linalg}
\begin{figure}[!h]
    \centering
    \includegraphics[width=.45\linewidth]{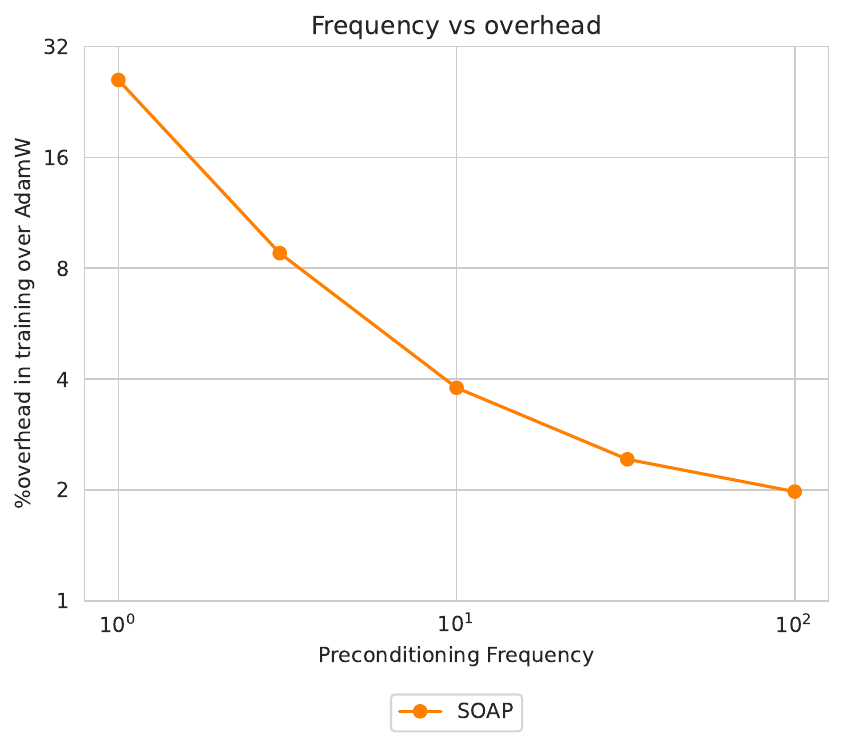}
    \includegraphics[width=.45\linewidth]{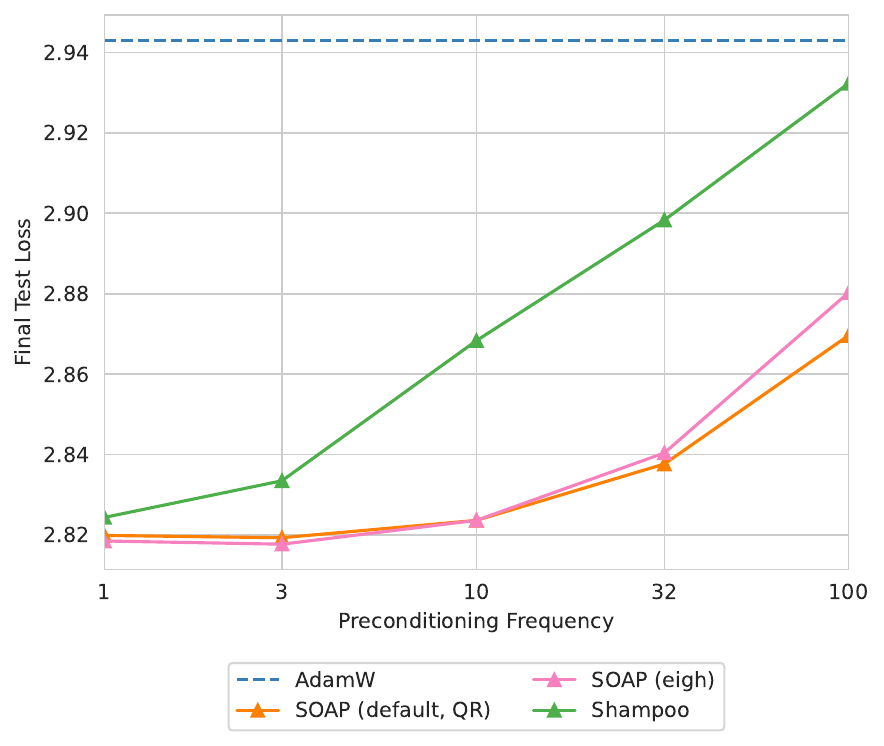}
    \caption{(Left) Depicting the overhead of SOAP over AdamW as a function of preconditioning frequency (Right) Comparing the performance of SOAP with \texttt{torch.linalg.eigh} for computing the eigenvectors with ~\Cref{alg:eigenvectors}, which uses \texttt{torch.linalg.qr}. Note that \texttt{torch.linalg.qr} is computationally more efficient than \texttt{torch.linalg.eigh} (as mentioned in \citet{eigh-vs-qr}); however, both seem to have comparable performance throughout the preconditioning frequency spectrum.}
    \label{fig:linalg-runtime}
\end{figure}

\textbf{Effect of frequency on  overhead:} In Figure~\ref{fig:linalg-runtime}
 (left), we observe that the overhead decreases as the preconditioning frequency increases, i.e., the frequency of invoking~\Cref{alg:eigenvectors}. If the only additional computation occurred in \Cref{alg:eigenvectors}, we would expect the overhead to scale as $1.0/(\text{preconditioning frequency})$, approaching zero. However, empirical results (Figure~\ref{fig:linalg-runtime} left) show that the overhead approaches an asymptote greater than zero. This is attributable to the additional matrix multiplications required to update $L$, update $R$, project the gradient, and reproject the gradient (for each layer) in the optimizer. Currently, these operations are performed in float32; reducing the precision of these operations, as proposed in \cite{4bitshampoo}, could lower this asymptote.

\subsubsection{Improving time overhead of SOAP}
\label{app:impr-time}
The per step overhead of SOAP can be reduced by using low precision to store the $L, R, Q_L, Q_R, V$ matrices~\citep{8bitadam,4bitshampoo}, which in turn will speed up computation done using these matrices. This approach cannot be used for reducing the overhead for the preconditioner update in popular deep learning frameworks such as Pytorch since $\texttt{torch.linalg.qr}$ does not support precision lower than $\texttt{float32}$. Orthogonal to the low precision approach we can improve the per step time overhead of SOAP by the following algorithmic approaches:
\begin{itemize}
    \item Using Adafactor instead of Adam (\Cref{app:adafactor}) as the diagonal preconditioner after rotating by $Q_L$ and $Q_R$. In this version of SOAP the overhead can be improved by from $m^3+n^3+2m^2n+2n^2m$ to $m^3+n^3+m^2n+n^2m+\max(m, n)^2\min(m, n) + \min(m, n)^3$ by merging the project and project back steps for the smaller dimension.
    \item Using one sided version of SOAP (\Cref{sec:one-sided}). This reduces overhead from $m^3+n^3+2m^2n+2n^2m$ to $\min(m, n)^3+2\min(m, n)^2\max(m, n)$. 
    \item Combining these approaches yields an overhead of $\min(m, n)^2\max(m, n)+2 \min(m, n)^3$
\end{itemize} 

Using one-sided version also reduces the second type of overhead from a calls to $\texttt{torch.linalg.qr}$ on a $m \times m$ and a $\ n \times n$ matrix to only a single call to $\min(m, n) \times \min(m, n)$ matrix.

\section{Discussion and Future Work}
We study an optimizer called SOAP: \textbf{S}hampo\textbf{O} with \textbf{A}dam in the \textbf{P}reconditioner's eigenbasis. We show that SOAP outperforms both AdamW and Shampoo in language modeling tasks and show that it is more robust to changes in preconditioning frequency than Shampoo. For future work, we would like to explore further improvements to the design of SOAP, in particular, related to using lower precision for the preconditioners as well as a better distributed implementation. We would also like to explore the performance of SOAP on other domains such as vision.

\section{Discussion and Limitations}
We study an optimizer called SOAP: \textbf{S}hampo\textbf{O} with \textbf{A}dam in the \textbf{P}reconditioner's eigenbasis. We show that SOAP outperforms both AdamW and Shampoo in language modeling tasks and show that it is more robust to changes in preconditioning frequency than Shampoo. While we have explored many factors such as batch size~(\Cref{sec:critical}) and training duration~(\Cref{sec:long}) we acknowledge that our study focuses on a relatively small scale compared to recent LLMs~\cite{llama} which are two orders of magnitude bigger. We hypothesize that our findings on the performance of SOAP would generalize to larger scales due to its theoretical foundation. SOAP's robustness is also supported by the fact that SOAP is equivalent to running Adam in a rotated space, and Adam has proven to be effective across scale and tasks. However, this hypothesis remains to be validated. %

For future work, we aim to improve the design of SOAP further, particularly by exploring the use of lower precision for preconditioners and optimizing its distributed implementation. Additionally, we are interested in testing SOAP’s performance in other domains, such as vision, to evaluate its performance across different types of tasks.

\section*{Acknowledgments}

SK, DM, and RZ acknowledges support from the Office of Naval Research under award N0001422-1-2377 and the National Science Foundation Grant under award \#IIS 2229881. This work has been made possible in part by a gift from the Chan Zuckerberg Initiative Foundation to establish the
Kempner Institute for the Study of Natural and Artificial Intelligence. NV, DM and RZ are supported
by a Simons Investigator Fellowship, NSF grant DMS-2134157, DARPA grant W911NF2010021,and
DOE grant DE-SC0022199. LJ acknowledges funding from the National Science Foundation
DMS-2134157.

\bibliography{ref}
\bibliographystyle{iclr2024_conference}

\appendix

\section{Experimental Setup}
\label{app:setup}

Many aspects of our setup such as models are the same as in~\citet{zhaoscience}. We will restate those details verbatim for completeness.

We train language models on C4 tokenized with the T5 tokenizer \citep{raffel2020exploring} and report results in terms of validation loss.

\paragraph{Models.} We start from the OLMo codebase \citep{groeneveld2024olmo} and train decoder-only transformer models of three sizes: 210m, 360m, and 660m, where the parameter count refers to non-embedding parameters. The models have widths of 1024, 1024, and 1408 and depths of 12, 24, 24. We used the 210m model to explore various ablations, most of our reported results are on 360m and 660m. The MLP hidden dimension is 4x of the width. The activation function is GeLU \citep{hendrycks2016gaussian}. We use RoPE positional encodings \citep{su2024roformer}. Attention heads are always dimension 64. We use PyTorch default LayerNorm. We use QK layer norm~\citep{dehghani2023scaling}. Following \citet{wortsman2024smallscale} we do not learn biases for the linear layers or LayerNorms. We train in mixed precision with bfloat16.

\paragraph{Algorithms.} We use the standard Pytorch implementation of AdamW \citep{paszke2019pytorch}, the DistributedShampoo~\cite{distributedshampoo}  implementation of Shampoo. We implement ourselves SOAP and GaLore starting from an older version of Pytorch implementation of AdamW and the official GaLore implementation~\cite{galoregithub}.

\paragraph{Default hyperparameters.} We use $\beta_1 = 0.95$, as we found it to outperform $\beta_1 = 0.9$ in our sweeps for the 360m model. Following~\citet{wortsman2024smallscale} we use decoupled weight decay with coefficient \num{1e-4} and z-loss with coefficient \num{1e-4}. We use the default value of $\eps = 1e-8$ in AdamW (actual or when used for grafting), SOAP and GaLore. We use warmup followed by cosine decay as our scheduler.  We start the warmup and end the cosine decay at $0.1x$ the maximum learning rate.

\paragraph{Default hyperparameters for DistributedShampoo} \citet{distributedshampoo} state that they find the optimal exponent to be either $-1/2$ or $-1.82/4 \approx -1/2.2$. Our preliminary findings were similar to this. Hence we set the default values of exponent to be $-1/2.5$ for both 1D and 2D parameters. We set $\eps_{\text{shampoo}} = \num{1e-12}$ and $\beta_{\text{shampoo}} = 0.95$ based on our initial set of experiments on the 210m model.

\paragraph{Default hyperparameters for GaLore} GaLore introduces an additional hyperparameter called scale ($\alpha$) since due to low rank updates the overall update magnitude decreases. Since we are running a full rank version of GaLore we set $\alpha = 1$.

\paragraph{Token counts.} For all of our runs we use a sequence length of 1024. For all models (except in \Cref{sec:critical}), we use a token batch size of 2048k $\approx$ 2m. We default to training models for the approximately ``chinchilla optimal'' number of tokens that is $ \approx$20 times the number of parameters. Explicitly, this means for our default batch size of 2m, the 210m models are trained for 1600 steps or $\approx$ 3.3b tokens. The 360m models are trained for 3200 steps, the 660m models are trained for 6400 steps.

\subsection{Sweeping over hyperparameters}

\textbf{AdamW, 2m batch size:} Starting from the default hyperparameters above we do the following sweeps:
\begin{enumerate}
    \item We sweep over learning rate in $\{.1, .0316, .01,\ldots, \num{3.16e-4}\}$.
    \item (360m) We sweep over the cross product of best 3 learning rates and $\beta_1 \in \{0.9, 0.95, 0.99\}$.
    \item (360m) We sweep over the cross product of best 3 learning rates and $\beta_2 \in \{0.9, 0.95, 0.99\}$.
\end{enumerate}

The last two of the sweeps did not yield any benefit for the 360m model with 2m batch size hence we only sweep over learning rate for the 660m model with 2m batch size.

\textbf{DistributedShampoo, 2m batch size:} Starting from the default hyperparameters above we do the following sweeps:
\begin{enumerate}
    \item We sweep over learning rate in $\{.1, .0316, .01,\ldots, \num{3.16e-4}\}$.
    \item (360m) We sweep over over the cross product of best 3 learning rates from above and $\eps_{\text{shampoo}} \in \{\num{1e-11}, \num{1e-12}, \num{1e-13}\}$.
    \item (360m) We sweep over over the cross product of best 3 learning rates from above and $\beta_{\text{shampoo}} \in \{.9, .95, .975\}$.
    \item Let $e_1, e_2$ denote the exponents used in DistributedShampoo for 1D and 2D parameters respectively. We also sweep over the cross product of best 3 learning rates from above and $(e_1, e_2)$ in $\{(2, 2), (2.5, 2.5), (3, 3), (2, 4)\}$.
\end{enumerate}

These sweeps did not yield any significant improvement in performance ($<.004$) for the 360m model. Hence we only sweep over the learning rate for the 660m model.

\textbf{SOAP, 2m batch size:} Starting from the default hyperparameters above we sweep over learning rate in $\{.1, .0316, .01,\ldots, \num{3.16e-4}\}$.

\textbf{AdamW, 256k batch size:}
For the 360m model with 256 batch size we start from the default hyperparameters and do the following sweeps:
\begin{enumerate}
    \item We sweep over learning rate in $\{.1, .0316, .01,\ldots, \num{3.16e-4}\}$.
    \item We sweep over the cross product of best 3 learning rates and $\beta_2 \in \{0.95,0.99\}$.
\end{enumerate}

In the second sweep we observe small improvements in performance by using $\beta_2 = .99$, so our final numbers use $\beta_2 = .99$. This (small) improvement in performance by using a larger $\beta_2$ at smaller batch sizes was also observed by~\citet{porian,zhaoscience}.

\textbf{DistributedShampoo, 256k batch size:}
For the 360m model with 256 batch size we start from the default hyperparameters and do the following sweeps:
\begin{enumerate}
    \item We sweep over learning rate in $\{.1, .0316, .01,\ldots, \num{3.16e-4}\}$.
    \item We sweep over the cross product of best 3 learning rates and $(\beta_2, \beta_{\text{shampoo}}) \in \{(.95, .95), (.99, .99)\}$.
\end{enumerate}

In the second sweep we observe small improvements in performance by using $\beta_2 = \beta_{\text{shampoo}} = .99$, so our final numbers use $\beta_2 = \beta_{\text{shampoo}} = .99$. 

\textbf{SOAP, 256k batch size:}
For the 360m model with 256 batch size we start from the default hyperparameters and do the following sweeps:
\begin{enumerate}
    \item We sweep over learning rate in $\{.1, .0316, .01,\ldots, \num{3.16e-4}\}$.
    \item We sweep over the cross product of best 3 learning rates and $\beta_2 \in \{.95, .99\}$.
\end{enumerate}

In the second sweep we observe small improvements in performance by using $\beta_2 = .99$, so our final numbers use $\beta_2 = .99$. 

\textbf{Preconditioning frequency sweeps:} For the preconditioning frequency experiments of SOAP and Shampoo (~\Cref{fig:main} (right)), for each frequency we do a learning rate sweep over the best 3 learning rates found at preconditioning frequency 10. Other hyperparameters are set to their optimal values obtained using the precondition frequency 10 sweeps.

\textbf{360m and 660m shorter runs:} For each of the shorter runs of 360m and 660m models for the SOAP optimizer (\Cref{fig:prec_overhead}), we did learning rate sweep over the best 3 learning rates found for the standard length run. Other hyperparameters are set to their optimal values obtained using the standard length run.

\textbf{Warmup:} The warmup duration for the 360m and 660m models were 600 and 1200 steps respectively. For the shorter runs (\Cref{fig:prec_overhead}), for 360m model, the warmup durations were 400, 400, 500 and 525 steps for 0.5, 0.625, 0.75 and 0.875 runs respectively. For the 660m model, the warmup durations were 600, 750, 900 and 1050 steps for 0.5, 0.625, 0.75 and 0.875 runs respectively. For 360m model with 256k batch size (\Cref{sec:critical}) we use a warmup for 4000 steps (total steps is 25000).

\section{GaLore}
\label{app:galore}

We tried GaLore for 210m model, and while it outperformed AdamW it performed worse than Shampoo. Hence we do not try GaLore for higher model sizes.

\textbf{Hyperparameter sweeps:} We did the following sweeps:
\begin{enumerate}
    \item We swept the cross product over learning rate ($3.16e-4, 1e-3, 3.16e-3, 1e-2$), preconditioning frequency ($10, 50, 200$), both sided and one sided versions. Frequency 200 had the best results matching the observation of~\citet{galore}.
    \item We did a cross product sweep over learning rate ($3.16e-4, 1e-3, 3.16e-3, 1e-2$), both sided and one sided versions with $\beta_2 = .99$ instead of $.95$ and preconditioning frequency 200.
    \item We did a cross product sweep over learning rate ($3.16e-4, 1e-3, 3.16e-3, 1e-2$), both sided and one sided versions, preconditioning frequency ($50, 200$)  with $\beta_1 = .9$ instead of $.95$.
\end{enumerate}

The best performing run among all of these achieved a final loss of 3.12 while the best Shampoo run achieved a final loss of 3.10.

\end{document}